\documentclass{article}

\usepackage{microtype}
\usepackage{times}
\usepackage{graphicx} 
\usepackage{subfigure} 
\usepackage{natbib}
\usepackage{amsmath} 
\usepackage{amsfonts} 
\usepackage{amssymb}
\usepackage{amsthm}
\usepackage{algorithm}
\usepackage{algorithmic}
\usepackage{hyperref}

\usepackage[accepted]{icml2015} 

\graphicspath{ {figures/} }

\newtheorem{theorem}{Theorem}[section]
\newtheorem{lemma}[theorem]{Lemma}
\newtheorem{prop}[theorem]{Proposition}

\newcommand{\ip}[2]{\langle {#1},\, {#2} \rangle}

\newcommand{\nlsum}{\sum\nolimits}

\newcommand{\C}{\mathbb{C}}

\newcommand{\set}[1]{\{ #1\}}
\newcommand{\invp}[1]{\left(#1\right)^{-1}}
\newcommand{\inv}[1]{#1^{-1}}

\newcommand{\Pc}{\mathcal{P}}

\newcommand{\Yc}{\mathcal{Y}}

\newcommand{\scal}[2]{\left\langle#1\,\middle\vert\,#2\right\rangle}

\DeclareMathOperator{\argmax}{argmax}
\numberwithin{equation}{section}

\definecolor{dkgreen}{rgb}{0,0.6,0}
\definecolor{gray}{rgb}{0.5,0.5,0.5}
\definecolor{mauve}{rgb}{0.58,0,0.82}

\icmltitlerunning{Fixed-point algorithms for determinantal point processes}

\begin{document} 

\twocolumn[
\icmltitle{Fixed-point algorithms for learning determinantal point processes}
\icmlauthor{Zelda Mariet}{zelda@csail.mit.edu}
\icmlauthor{Suvrit Sra}{suvrit@mit.edu}
\icmladdress{Massachusetts Institute of Technology, Cambridge, MA 02139 USA}
\icmlkeywords{manifold optimization; fixed-point theory; DPP}

\vskip 0.3in
]

\begin{abstract}
  Determinantal point processes (DPPs) offer an elegant tool for encoding probabilities over subsets of a ground set. Discrete DPPs are parametrized by a positive semidefinite matrix (called the DPP kernel), and estimating this kernel is key to learning DPPs from observed data.  We consider the task of learning the DPP kernel, and develop for it a surprisingly simple yet effective new algorithm. Our algorithm offers the following benefits over previous approaches: (a) it is much simpler; (b) it yields equally good and sometimes even better local maxima; and (c) it runs an order of magnitude faster on large problems. We present experimental results on both real and simulated data to illustrate the numerical performance of our technique.
\end{abstract}

\section{Introduction}
Determinantal point processes (DPPs) arose in statistical mechanics, where they were originally used to model fermions~\citep{macchi}. Recently, they have  witnessed substantial interest in a variety of machine learning applications~\citep{kulesza,kulTas.book}.

One of the key features of DPPs is their ability to model the notion of diversity while respecting quality, a concern that underlies the broader task of subset selection where balancing quality with diversity is a well-known issue---see e.g., document summarization~\citep{linBilmes}, object retrieval~\citep{affandi}, recommender systems~\citep{zhou}, and sensor placement~\citep{krause}.

DPPs are also interesting in their own right: they have various combinatorial, probabilistic, and analytic properties, while involving a fascinating set of open problems~\citep{lyons,hough,kulesza}. 

Within machine learning DPPs have found good use---see for instance~\citep{gillen}; \citep{kulTas11.uai}; \citep{kulTas11.icml}; \citep{affandi}; \citep{affTas}; \citep{affKulTas}; \citep{gillen12}. For additional references and material we refer the reader to the survey~\citep{kulTas.book}.

Our paper is motivated by the recent work of~\citet{gillen}, who made notable progress on the task of learning a DPP kernel from data. This task is conjectured to be NP-Hard~\citep[Conjecture 4.1]{kulesza}. \citet{gillen} presented a carefully designed EM-style procedure, which, unlike several previous approaches (e.g., \citep{kulTas11.uai,kulTas11.icml,affandi}) learns a full DPP kernel nonparameterically.  

One main observation of~\citet{gillen} is that applying projected gradient ascent to the DPP log-likelihood usually results in degenerate estimates (because it involves projection onto the set $\set{X : 0 \preceq X \preceq I}$). Hence one may wonder if instead we could apply more sophisticated manifold optimization techniques~\citep{absil,manopt}. While this idea is attractive, and indeed applicable, e.g., via the excellent \textsc{Manopt} toolbox~\citep{manopt}, empirically it turns out to be computationally too demanding; the EM strategy of \citet{gillen} is more practical. 

We depart from both EM and manifold optimization to develop a new learning algorithm that (a) is simple, yet powerful; and (b) yields essentially the same log-likelihood values as the EM approach while running significantly faster. In particular, our algorithm runs an order of magnitude faster on larger problems. 

The key innovation of our approach is a derivation via a fixed-point view, which by construction ensures positive definiteness at every iteration. Its convergence analysis involves an implicit  bound-optimization iteration to ensure monotonic ascent.\footnote{The convergence analysis in this version of the paper improves upon our original submission, in that our proof is now constructive and requires weaker assumptions.} A pleasant byproduct of the fixed-point approach is that it avoids any eigenvalue/vector computations, enabling a further savings in running time.

\subsection{Background and problem setup}
Without loss of generality we assume that the ground set of $N$ items is $\set{1,2,\ldots,N}$, which we denote by $\Yc$. A (discrete) DPP on $\Yc$ is a probability measure $\Pc$ on $2^{\Yc}$ (the set of all subsets of $\Yc$) such that for any $Y \subseteq \Yc$, the probability $\Pc(Y)$ verifies $\Pc(Y) \propto \det(L_Y)$; here $L_Y$ denotes the principal submatrix of the \emph{DPP kernel} $L$ induced by indices in $Y$. Intuitively, the diagonal entry $L_{ii}$ of the kernel matrix $L$ captures some notion of the importance of item $i$, whereas an off-diagonal entry $L_{ij}=L_{ji}$ measures similarity between  items $i$ and $j$. This intuitive notion provides further motivation for seeking DPPs with non-diagonal kernels when there is implicit interaction between the observed items. 

The normalization constant for the measure $\Pc$ follows upon observing that $\sum_{Y \subseteq \Yc} \det(L_Y) = \det(L+I)$. Thus,
\begin{equation}
  \label{eq:1}
  \Pc(Y) = \frac{\det(L_Y)}{\det(L+I)},\qquad Y \subseteq \Yc.
\end{equation}

DPPs can also be given an alternative representation through a \emph{marginal kernel} $K$ that captures for a random $Y \sim \Pc$ and every $A \subseteq \Yc$, the marginal probability
\begin{equation}
  \label{eq:2}
  \Pc(A \subseteq Y) = \det(K_A).
\end{equation}
It is easy to verify that $K = L(L+I)^{-1}$, which also implies that $K$ and $L$ have the same eigenvectors and differ only in their eigenvalues. It can also be shown~\citep{kulesza} that $\Pc(Y) = |\det(K-I_{Y^c})|$, where $I_{Y^c}$ is a partial $N\times N$ identity matrix with diagonal entries in $Y$ zeroed out. 

Both parameterizations~\eqref{eq:1} and~\eqref{eq:2} of the DPP probability are useful: \citet{gillen} used a formulation in terms of $K$; we prefer~\eqref{eq:1} as it aligns better with our algorithmic approach.

\subsection{Learning the DPP Kernel}
The learning task aims to fit a DPP kernel (either $L$ or equivalently the marginal kernel $K$) consistent with a collection of observed subsets. Suppose we obtain as training data $n$ subsets $(Y_1,\ldots,Y_n)$ of the ground set $\Yc$. The task is to maximize the likelihood of these observations. Two equivalent formulations of this maximization task may be considered:
\begin{align}
  \label{eq:3}
  \max_{L \succeq 0}\ &\nlsum_{i=1}^n \log\det(L_{Y_i}) - n\log\det(I+L),\\
  \label{eq:4}
  \max_{0 \preceq K \preceq I}\ &\nlsum_{i=1}^n \log\bigl(|\det(K-I_{Y_i^c})| \bigr).
\end{align}
We will use formulation~\eqref{eq:3} in this paper. \citet{gillen} used~\eqref{eq:4} and exploited its structure to derive a somewhat intricate EM-style method for optimizing it. Both \eqref{eq:3} and \eqref{eq:4} are nonconvex and difficult optimize. For instance, using projected gradient on~\eqref{eq:4} may seem tempting, but projection ends up yielding degenerate (diagonal and rank-deficient) solutions which is undesirable when trying to capture interaction between observations---indeed, this criticism motivated~\citet{gillen} to derive the EM algorithm.

We approach problem \eqref{eq:3} from a different viewpoint (which also avoids projection) and as a result obtain a new optimization algorithm for estimating $L$. This algorithm, its analysis, and empirical performance are the subject of the remainder of the paper.

\section{Optimization algorithm}
\label{sec:algo}
The method that we derive has two key components: (i) a fixed-point view that helps obtain an iteration that satisfies the crucial positive definiteness constraint $L\succeq 0$ by construction; and (ii) an implicit bound optimization based analysis that ensures monotonic ascent. The resulting algorithm is vastly simpler than the previous EM-style approach of~\citet{gillen}.

If $|Y|=k$, then for a suitable $N\times k$ indicator matrix $U$ we can write $L_Y = U^*LU$, which is also known as a \emph{compression} ($U^*$ denotes the Hermitian transpose). We write $U_i^*LU_i$ interchangeably with $L_{Y_i}$, implicitly assuming suitable indicator matrices $U_i$ such that $U_i^*U_i = I_{|Y_i|}$. To reduce clutter, we will drop the subscript on the identity matrix, its dimension being clear from context. 

Denote by $\phi(L)$ the objective function in~\eqref{eq:3}. Assume for simplicity that the constraint set is open, i.e., $L \succ 0$. Then any critical point of the log-likelihood must satisfy
\begin{equation}
  \label{eq:5}
  \begin{split}
  &\nabla\phi(L) = 0,\quad\text{or equivalently}\\
  &\nlsum_{i=1}^nU_i\invp{U_i^*LU_i}U_i^* - n\invp{I+L} = 0.
\end{split}
\end{equation}
Any (strictly) positive definite solution to the nonlinear matrix equation~\eqref{eq:5} is a candidate locally optimal solution.

We solve this matrix equation by developing a fixed-point iteration. In particular, define \[\Delta := \tfrac{1}{n}\nlsum_{i=1}^nU_i\invp{U_i^*LU_i}U_i^* - (I+L)^{-1},\] with which we may equivalently write~\eqref{eq:5} as
\begin{equation}
  \label{eq:6}
  \Delta + \inv{L} = \inv{L}.
\end{equation}
Equation~\eqref{eq:6} suggests the following iteration
\begin{equation}
  \label{eq:7}
  L_{k+1}^{-1} \gets \inv{L_k} + \Delta_k,\qquad k=0,1,\ldots.
\end{equation}
\emph{A priori} there is no reason for iteration~\eqref{eq:7} to be valid (i.e., converge to a stationary point). But we write it in this form to highlight its crucial feature: starting from an initial $L_0 \succ 0$, it generates positive definite iterates (Prop.~\ref{prop.pd}).

\begin{prop}
  \label{prop.pd}
  Let $L_0 \succ 0$. Then, the sequence $\set{L_k}_{k \ge1}$ generated by~\eqref{eq:7} remains positive definite.
\end{prop}
\begin{proof}
  The proof is by induction. It suffices to show that
  \begin{equation*}
    L \succ 0 \implies \inv{L} + \Delta \succ 0.
  \end{equation*}
  Since $I+L \succ L$, from the order inversion property of the matrix inverse map it follows that $\inv{L} \succ \inv{(I+L)}$.
  Now adding the matrix $\tfrac1n\sum_{i=1}U_i\invp{U_i^*LU_i}U_i^* \succeq 0$ we obtain the desired inequality by definition of $\Delta$.
\end{proof}

A quick experiment reveals that iteration~\eqref{eq:7} \emph{does not converge} to a local maximizer of $\phi(L)$. To fix this defect, we rewrite the key equation~\eqref{eq:6} in a different manner:
\begin{equation}
  \label{eq:8}
  L = L + L\Delta L.
\end{equation}
This equation is obtained by multiplying~\eqref{eq:6} on the left and right by $L$. Therefore, we now consider the iteration
\begin{equation}
  \label{eq:9}
  L_{k+1} \gets L_k + L_k\Delta_kL_k,\quad k = 0,1,\ldots.
\end{equation}
Prop.~\ref{prop.pd} in combination with the fact that congruence preserves positive definiteness (i.e., if $X \succeq 0$, then $Z^*XZ \succeq 0$ for any complex matrix $Z$), implies that if $L_0 \succ 0$, then the sequence  $\set{L_k}_{k\ge 1}$ obtained from iteration~\eqref{eq:9} is also positive definite. What is more remarkable is that contrary to iteration~\eqref{eq:7}, the sequence generated by~\eqref{eq:9} monotonically increases the log-likelihood. 

While monotonicity is not apparent from our intuitive derivation above, it  becomes apparent once we recognize an implicit change of variables that seems to underlie our method.

\subsection{Convergence Analysis}
\begin{theorem}
  \label{thm:convergence}
  Let $L_k$ be generated via~\eqref{eq:9}. Then, the sequence $\{\phi(L_k)\}_{k \geq 0}$ is monotonically increasing.
\end{theorem}
Before proving Theorem~\ref{thm:convergence} we need the following lemma.
\begin{lemma}
  \label{lemma:convexity}
  Let $U\in \C^{N\times k}$ ($k \le N$) such that $U^*U=I$. The map
  $g(S) := \log \det (U^* S^{-1} U)$ is convex on the set of positive definite matrices.
\end{lemma}
\begin{proof}
Since $g$ is continuous it suffices to establish midpoint convexity. Consider therefore, $X,Y \succ 0$ and let \[X\#Y = X^{1/2}(X^{-1/2}YX^{-1/2})^{1/2}X^{1/2}\] be their geometric mean. The operator inequality $X\#Y \preceq \frac{X+Y}{2}$ is well-known~\citep[Thm.~4.1.3]{bhatia07}. Hence,
\begin{align*}
  \left(\tfrac{X+Y}{2}\right)^{-1} &\preceq (X\#Y)^{-1} = X^{-1}\#Y^{-1} \\
  U^*\left(\tfrac{X+Y}{2}\right)^{-1}U &\preceq U^*(X^{-1}\#Y^{-1}) U \\
  &\preceq (U^*X^{-1}U)\#(U^*Y^{-1}U),
\end{align*}
where equality follows from~\citep[Thm.~4.1.3]{bhatia07}, and the final inequality follows from~\citep[Thm.~4.1.5]{bhatia07}\footnote{For an explicit proof see~\citep[Thm.~8]{gopt}.}
Since $\log \det$ is monotonic on positive definite matrices and since $\det (A\#B) = \sqrt{\det A}\sqrt{\det B}$, it then follows that
\begin{equation*}
  \begin{split}
    \log \det\bigl( U^*\left(\tfrac{X+Y}{2}\right)^{-1}U\bigr) &\leq \tfrac 1 2 \log \det (U^*X^{-1}U) \\
    &+ \tfrac 1 2 \log \det (U^*Y^{-1}U),
\end{split}
\end{equation*}
which proves the lemma.
\end{proof}

Now we are ready to prove Theorem~\ref{thm:convergence}.
\begin{proof}[\textbf{Proof (Thm.~\ref{thm:convergence})}]
The key insight is to consider $S = L^{-1}$ instead of $L$; this change is only for the analysis---the actual iteration that we implement is still \eqref{eq:9}.\footnote{Our previous proof was based on viewing iteration~\eqref{eq:9} as a scaled-gradient-like iteration. However, we find the present version more transparent for proving monotonicity.}

Writing $\psi(S) := \phi(L)$, we see that $\psi(S)$ equals
\begin{align*}
   &\tfrac 1 n \nlsum_i \log \det (U_i^* S^{-1} U_i) - \log \det (S^{-1}+I)\\
   &= \log \det(S) + \tfrac 1 n \nlsum_i \log \det (U_i^* S^{-1} U_i) \\
  &\phantom{=~}-\log \det (I+S).
\end{align*}
Let $h(S) = \tfrac 1 n \sum_i \log \det (U_i^* S^{-1} U_i) - \log \det (I+S)$, and $f(S) = \log \det(S)$. Clearly, $f$ is concave in $S$, while  $h$ is convex is $S$; the latter from Lemma~\ref{lemma:convexity} and the fact that $-\log\det(I+S)$ is convex. This observation allows us to invoke iterative bound-optimization (an idea that underlies EM, CCCP, and other related algorithms). 

In particular, we construct an \emph{auxiliary function} $\xi$ so that
\begin{equation*}
  \begin{split}
    \psi(S) &\ge \xi(S,R),\qquad\forall S, R \succ 0,\\
    \psi(S) &=   \xi(S,S),\qquad\forall S \succ 0.
  \end{split}
\end{equation*}
As in~\citep{CCCP}, we select $\xi$ by exploiting the convexity of $h$: as $h(S) \ge h(R) + \ip{\nabla h(R)}{S-R}$, we simply set
\begin{equation*}
  \xi(S,R) := f(S) + h(R) + \scal{\nabla h(R)}{S-R}.
\end{equation*}
Given an iterate $S_k$, we then obtain $S_{k+1}$ by solving
\begin{equation}
  \label{eq:10}
  S_{k+1} := \argmax_{S \succ 0}\ \xi(S,S_k),
\end{equation}
which clearly ensures monotonicity: $\psi(S_{k+1}) \ge \psi(S_k)$.

Since~\eqref{eq:10} has an open set as a constraint and $\xi(S,\cdot)$ is strictly concave, to solve~\eqref{eq:10} it suffices to solve the necessary condition $\nabla_S\xi(S,S_k)=0$. This amounts to
\begin{align*}
  &S^{-1} =~(I+S_k)^{-1} +\tfrac 1 n \nlsum_i S_k^{-1} U_i (U_i^*S_k^{-1}U_i)^{-1}U_i^* S_k^{-1}.
\end{align*}
Rewriting in terms of $L$ we immediately see that with $L_{k+1} = L_k + L_k \Delta_k L_k$, $\phi(L_{k+1}) \ge \phi(L_k)$ (the inequality is strict unless $L_{k+1}=L_k$).
\end{proof}

Theorem~\ref{thm:convergence} shows that iteration~\eqref{eq:9} is well-defined (positive definiteness was established by Prop.~\ref{prop.pd}). The fixed-point formulation~\eqref{eq:9} actually suggests a broader iteration, with an additional step-size $a$:
\begin{equation}
  \label{eq:iter-general}
  L_{k+1}  = L_k + a L_k \Delta_k L_k.
\end{equation}
Above we showed that for $a=1$ ascent is guaranteed. Empirically, $a > 1$ often works well; Prop.~\ref{prop.abound} presents an easily computable upper bound on feasible $a$. We conjecture that for all feasible values $a \geq 1$, iteration~\eqref{eq:9} is guaranteed to increase the log-likelihood.

Moreover, all previous calculations can be redone in the context where $L = F^*WF$ for a fixed feature matrix $F$ in order to learn the weight matrix $W$ (under the assumption that $S^*S$ is invertible), making our approach also useful in the context of feature-based DPP learning.

Pseudocode of our resulting learning method is presented in Algorithms~\ref{algo:dpp} and \ref{algo:fixed-point}. For simplicity, we recommend using a fixed value of $a$ (which can be set at initialization).

\begin{algorithm}[ht]\small
   \caption{Picard Iteration}
   \label{algo:dpp}
\begin{algorithmic}
   \STATE {\bfseries Input:} Matrix $L$, training set $T$, step-size $a > 0$.
   \FOR{$i=1$ {\bfseries to} maxIter}
   \STATE $L \longleftarrow$ FixedPointMap($L$, $T$, $a$)
   \IF{stop($L$, $T$, $i$)} 
   \STATE break
   \ENDIF
   \ENDFOR

   \textbf{return} $L$
\end{algorithmic}
\end{algorithm}

\begin{algorithm}[ht]\small
   \caption{FixedPointMap}
   \label{algo:fixed-point}
\begin{algorithmic}
   \STATE {\bfseries Input:} Matrix $L$, training set $T$, step-size $a > 0$
   \STATE $Z \longleftarrow 0$
   \FOR{$Y$ {\bfseries in} $T$}
   \STATE $Z_Y = Z_Y + L_Y^{-1}$
   \ENDFOR

   \textbf{return} $L+a L (Z / |T| - (L +I)^{-1}) L$  
\end{algorithmic}
\end{algorithm}

\subsection{Iteration cost and convergence speed}
The cost of each iteration of our algorithm is dominated by the computation of $\Delta$, which costs a total of $O(\sum_{i=1}^n |Y_i|^3 + N^3) = O(n\kappa^3 + N^3)$ arithmetic operations, where $\kappa = \max_i|Y_i|$; the $O(|Y_i|^3)$ cost comes from the time required to compute the inverse $L_{Y_i}^{-1}$, while the $N^3$ cost stems from computing $\inv{(I+L)}$. Moreover, additional $N^3$ costs arise when computing $L\Delta L$. 

In comparison, each iteration of the method of \citet{gillen} costs $O(nN\kappa^2+N^3)$, which is comparable to, though slightly greater than $O(n\kappa^3+N^3)$ as $N \ge \kappa$. In applications where the sizes of the sampled subsets satisfy $\kappa \ll N$, the difference can be more substantial. Moreover, we do not need any eigenvalue/vector computations to implement our algorithm.

Finally, our iteration also runs slightly faster than the K-Ascent iteration, which costs $O(nN^3)$. Additionally, similarly to EM, our algorithm avoids the projection step necessary in the K-Ascent algorithm (which ensures $K \in \set{X : 0 \preceq X \preceq I}$). As shown in~\citep{gillen}, avoiding this step helps learn non-diagonal matrices. 

We note in passing that similar to EM, assuming a non-singular local maximum, we can also obtain a local linear rate of convergence. This follows by relating iteration~(\ref{eq:9}) to scaled-gradient methods~\citep[\S1.3]{bertsekas99} (except that we have an implicit PSD  constraint). 

\section{Experimental results}
\label{sec:expt}
We compare performance of our algorithm, referred to as \emph{Picard iteration}\footnote{Our nomenclature stems from the usual name for such iterations in fixed-point theory~\citep{granas03}.}, against the EM algorithm presented in~\citet{gillen}. We experiment on both synthetic\footnote{The figures and tables for the synthetic results have been modified to include some minor corrections: in particular, Tables~\ref{table:comp-basic} and~\ref{table:comp-wishart} now show the runtime to 99\%. The runtimes were initially to final convergence, but erroneously reported to be to 95\%.} and real-world data.

For real-world data, we use the baby registry test on which results are reported in~\citep{gillen}. This dataset consists in $111,006$ sub-registries describing items across 13 different categories; this dataset was obtained by collecting baby registries from \texttt{amazon.com}, all containing between 5 and 100 products, and then splitting each registry into subregistries according to which of the 13 categories (such as ``feeding'', ``diapers'', ``toys'', etc.) each product in the registry belongs to. \citep{gillen} provides a more in-depth description of this dataset.

These sub-registries are used to learn a DPP capable of providing  recommendations for these products: indeed, a DPP is well-suited for this task as it provides sets of products in a category that are popular yet diverse enough to all be of interest to a potential customer.

\subsection{Implementation details}
We measure convergence by testing the relative change $\frac {|\phi(L_{k+1}) - \phi(L_k)|} {|\phi(L_k)|} \leq \varepsilon$. We used a tighter convergence criterion for our algorithm ($\varepsilon_{\text{pic}} = 0.5 \cdot \varepsilon_{\text{em}}$) to account for the fact that the distance between two subsequent log-likelihoods tends to be smaller for the Picard iteration than for EM.

The parameter $a$ for Picard was set at the beginning of each experiment and never modified as it remained valid throughout each test\footnote{Although it was not necessary in our experiments, if the parameter $a$ becomes invalid, it can be halved until it reaches 1.}. In EM, the step size was initially set to 1 and halved when necessary, as per the algorithm described in~\citep{gillen}; we used the code of~\citet{gillen} for our EM implementation\footnote{These experiments were run with MATLAB, on a Linux Mint system, using 16GB of RAM and an i7-4710HQ CPU @ 2.50GHz.}.

\begin{figure*}[!ht]
  \centering
  \subfigure[$N = 50$]{\includegraphics[width=.3\textwidth]{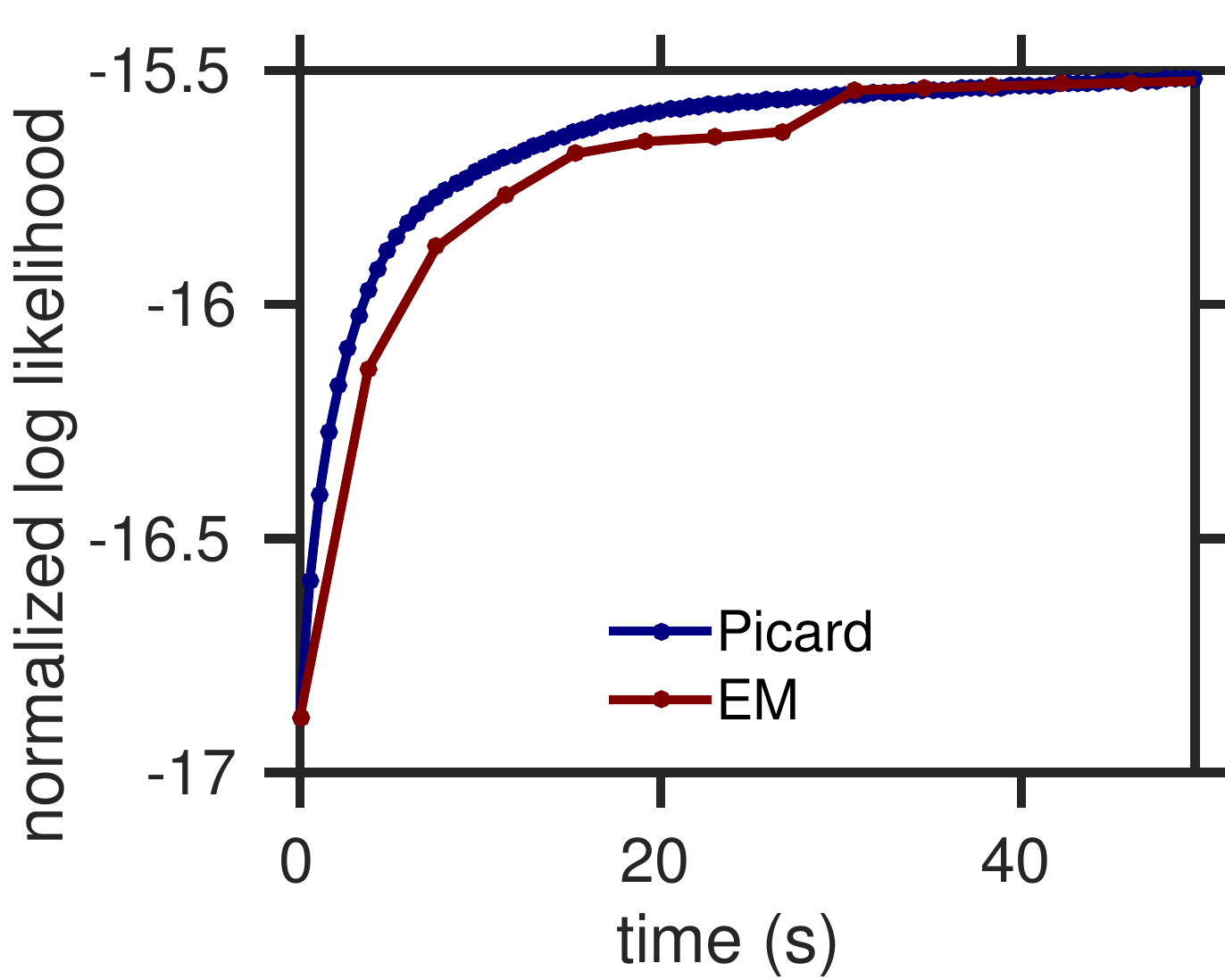}}
  \subfigure[$N = 100$]{\includegraphics[width=.3\textwidth]{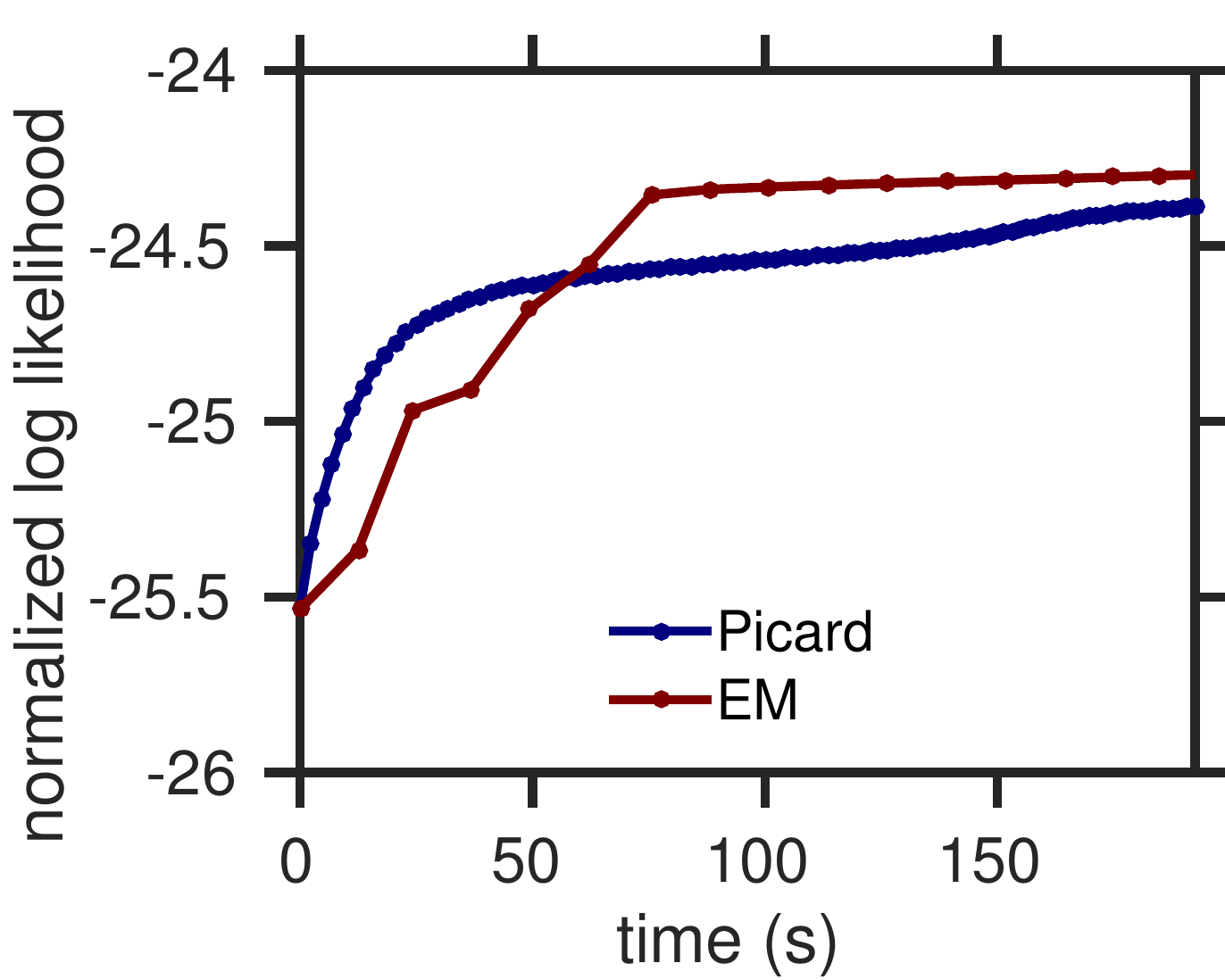}}
  \subfigure[$N = 150$]{\includegraphics[width=.3\textwidth]{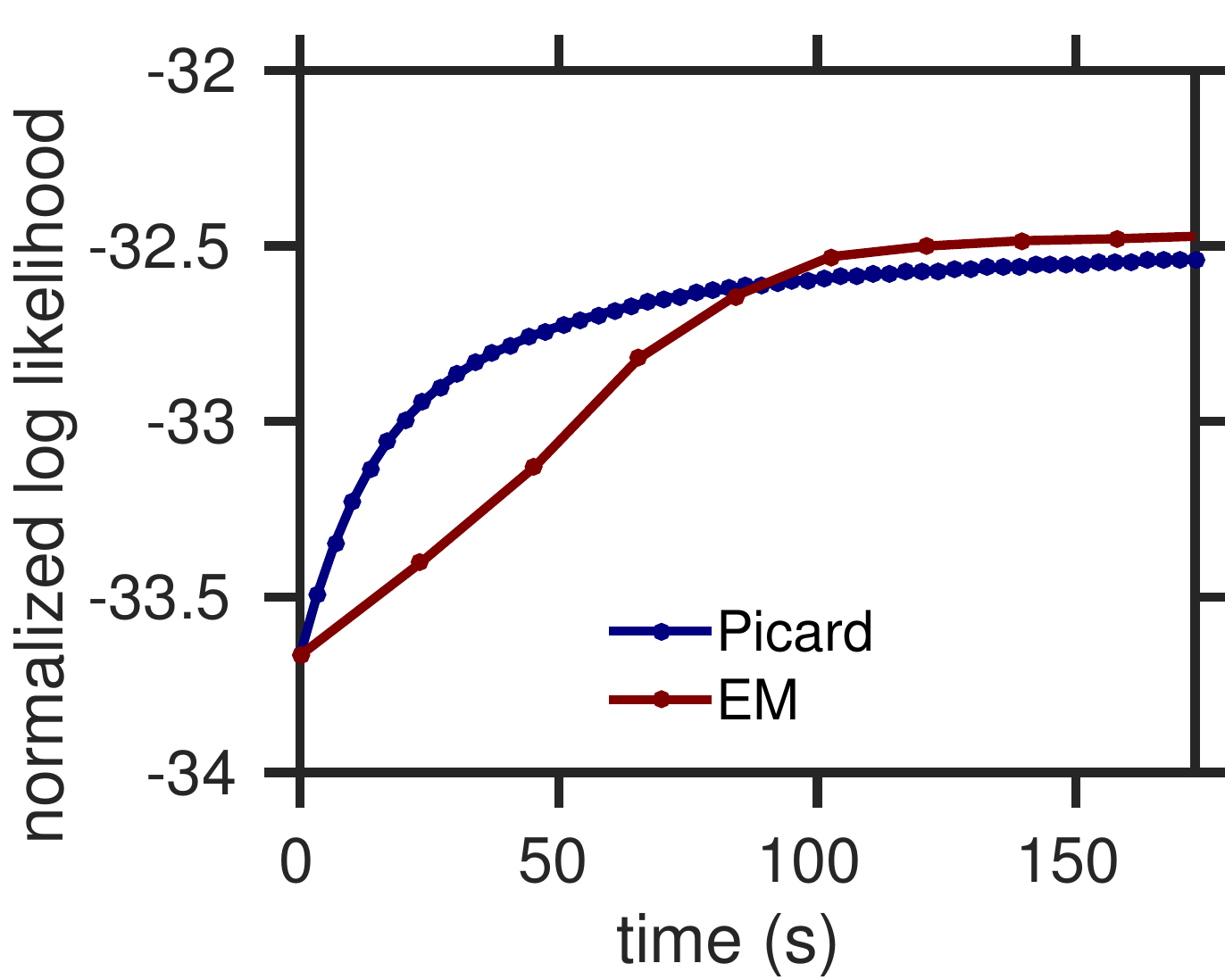}} 
  \caption{Normalized log-likelihood as a function of time for various set sizes $N$, with $n = 5000$ and $a = 5$ using the \texttt{BASIC} random distribution.}
  \label{fig:ll-time-set-size}
\end{figure*}
\begin{figure*}[!ht]
  \centering
  \subfigure[$n = 5000$]{\includegraphics[width=.3\textwidth]{temp_log_time_50_5000_5.pdf}}
  \subfigure[$n = 10,000$]{\includegraphics[width=.3\textwidth]{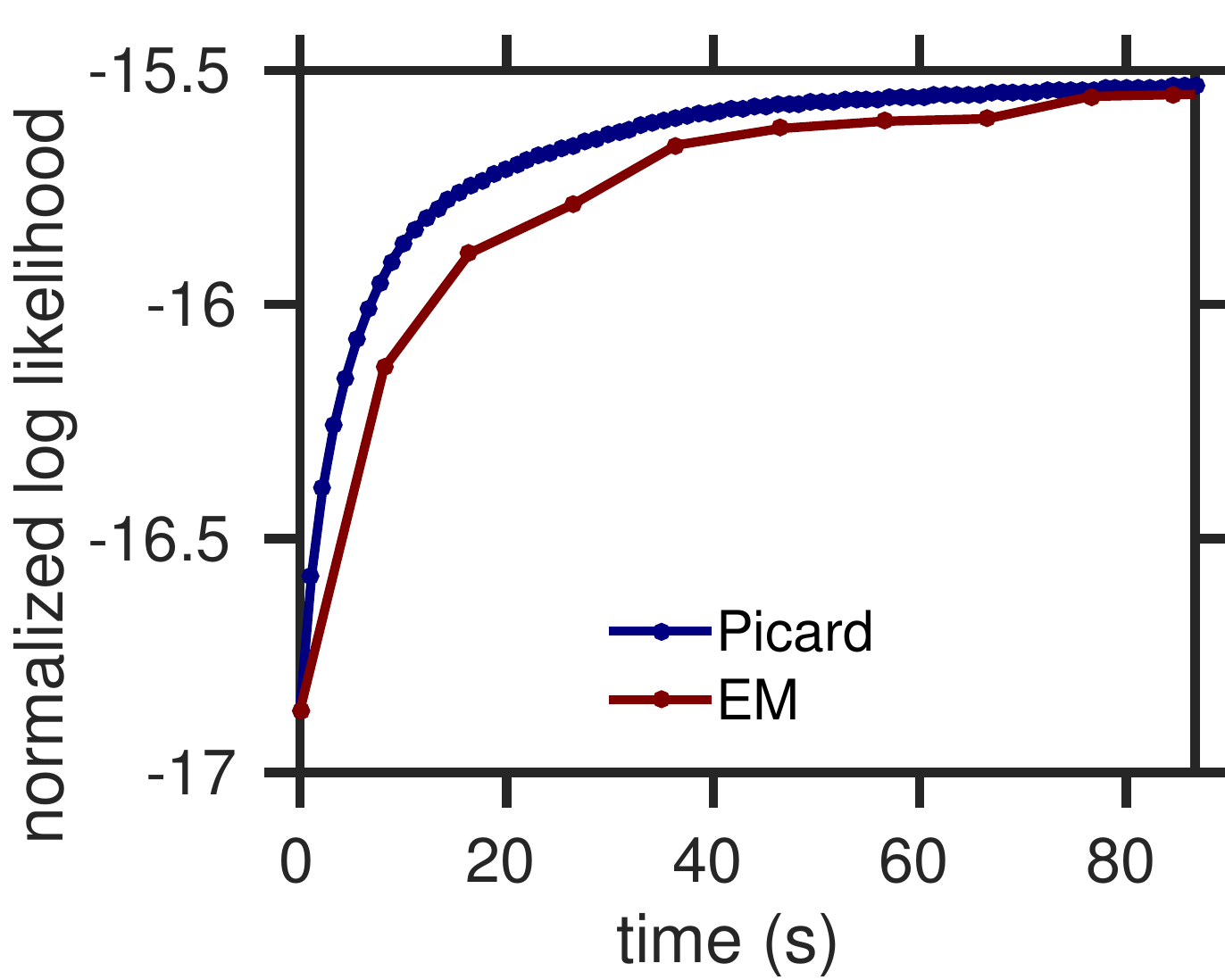}}
  \subfigure[$n = 15,000$]{\includegraphics[width=.3\textwidth]{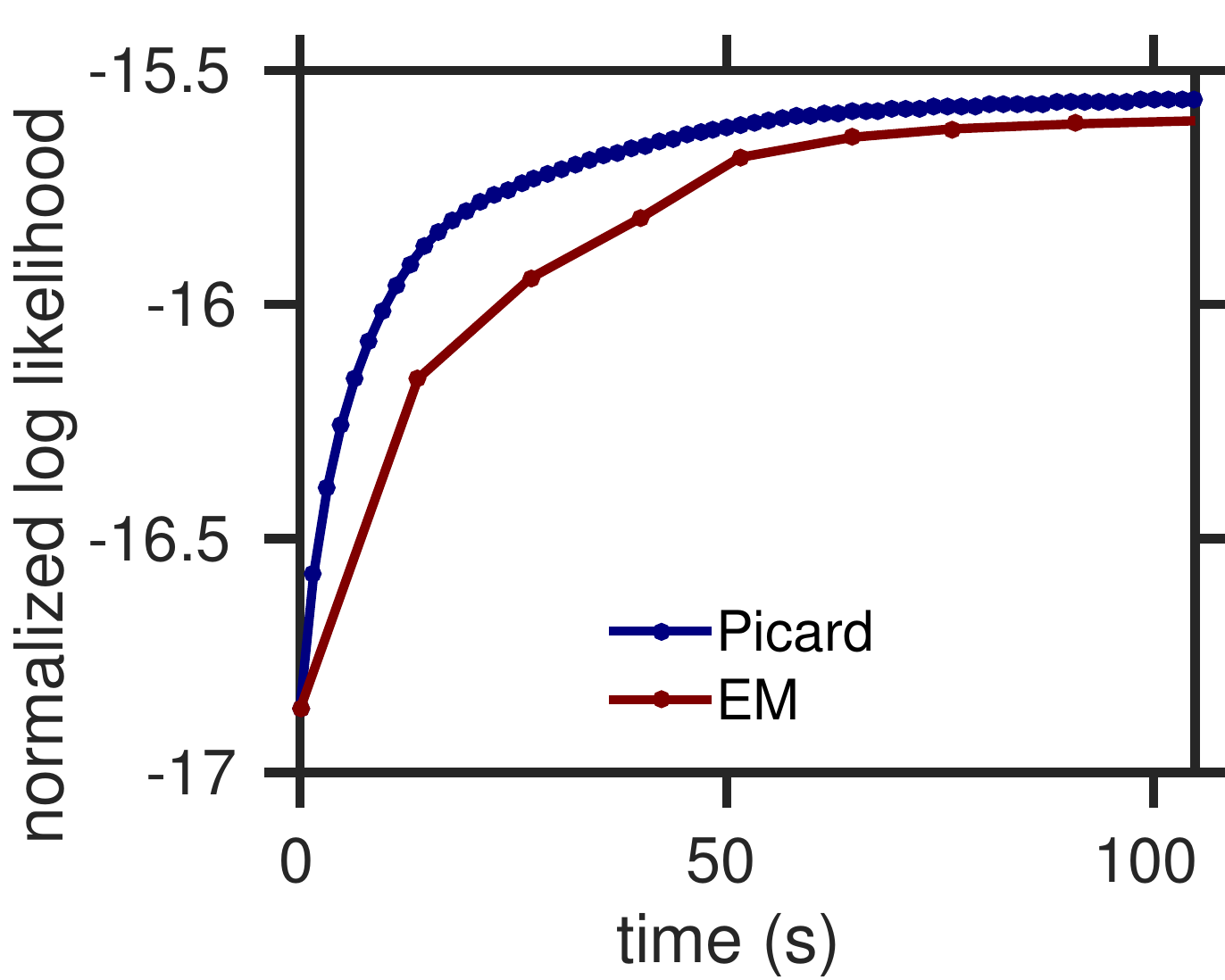}}
  \caption{Normalized log likelihood as a function of time for various numbers of training sets, with $N = 50$ and $a = 5$ using the \texttt{BASIC} random distribution.}
  \label{fig:ll-time-training-size}
\end{figure*}
\begin{figure*}[!ht]
  \centering
  \subfigure[$a = 1$]{\includegraphics[width=.3\textwidth]{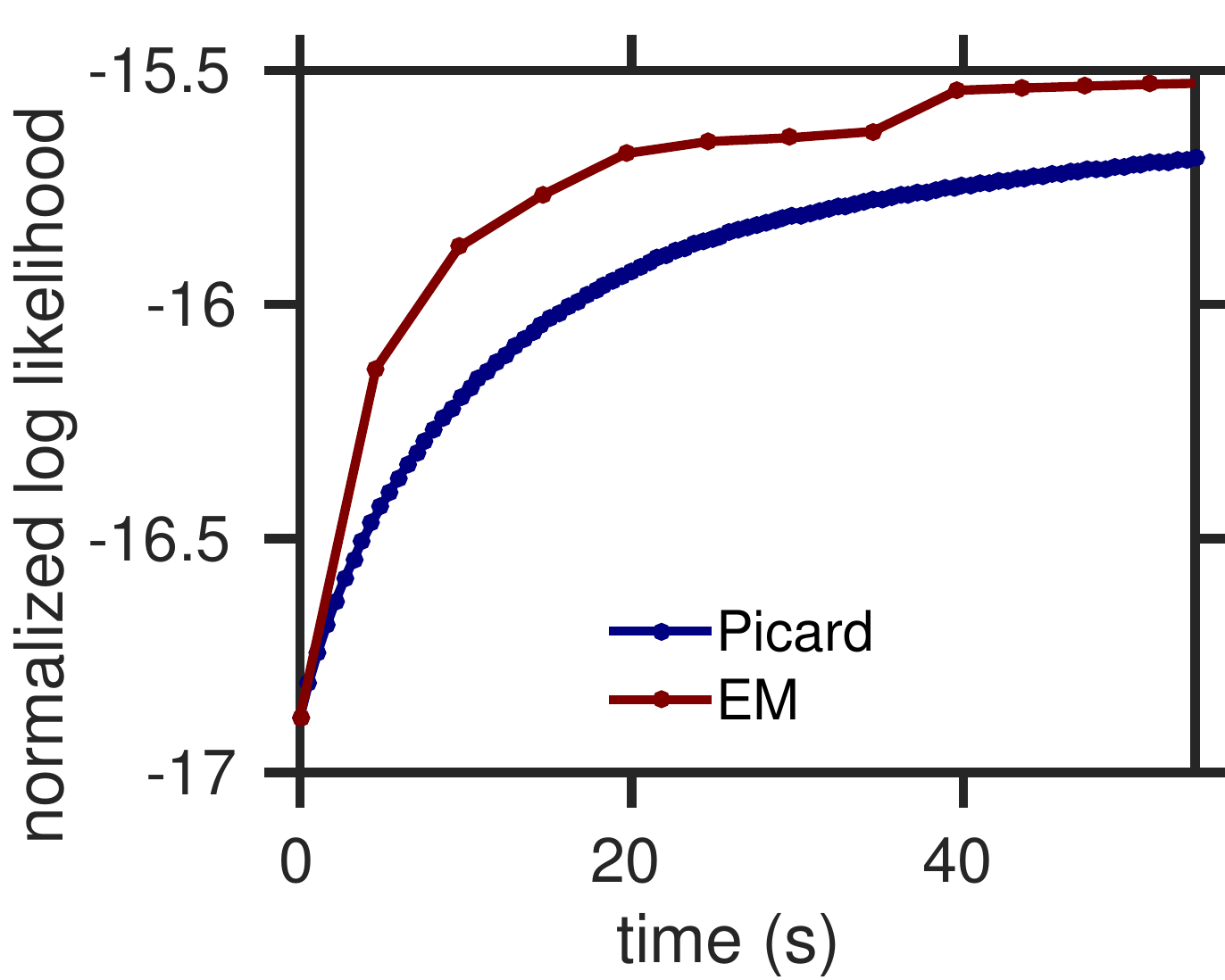}}
  \subfigure[$a = 5$]{\includegraphics[width=.3\textwidth]{temp_log_time_50_5000_5.pdf}}
  \subfigure[$a = 10$]{\includegraphics[width=.3\textwidth]{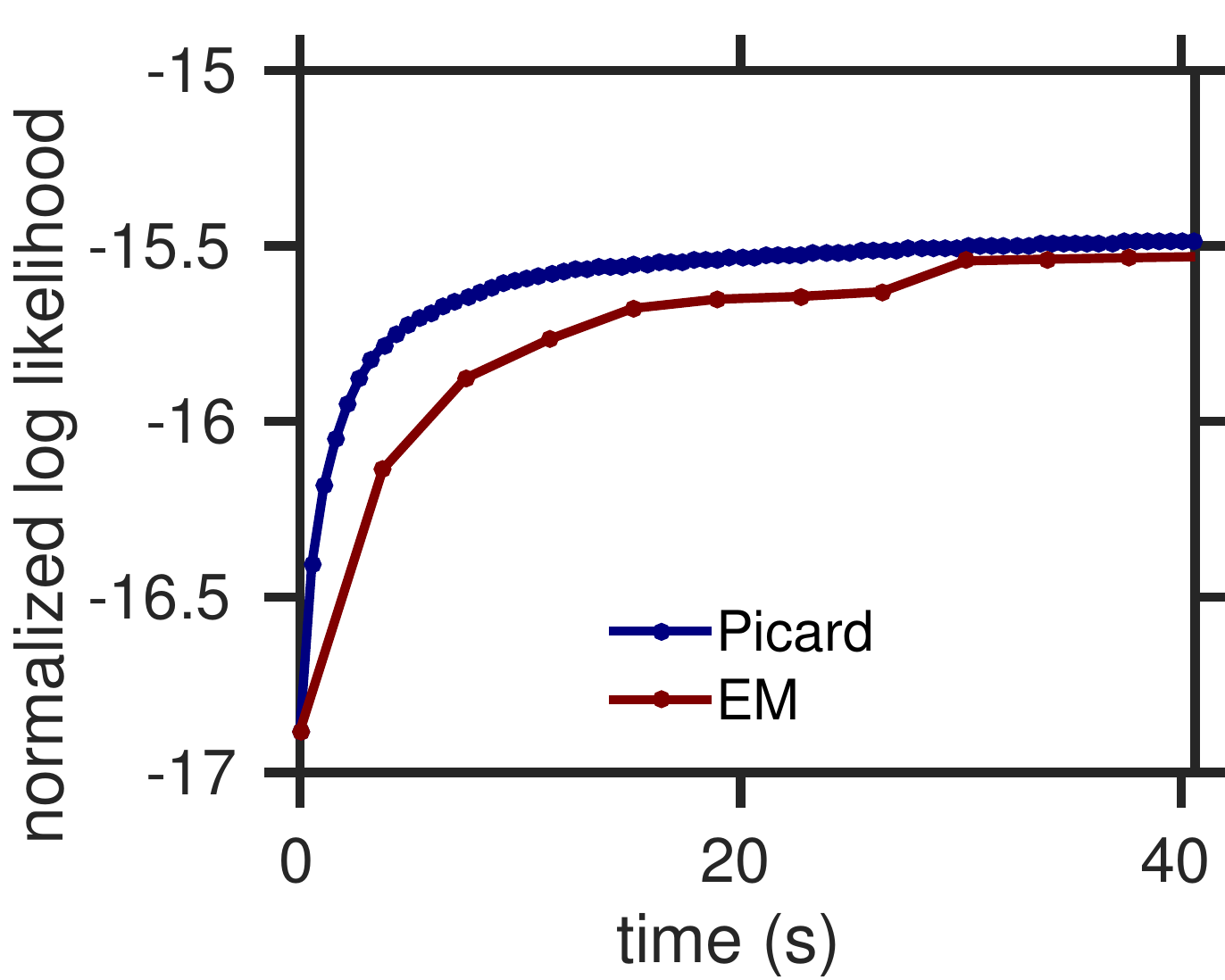}}
  \caption{Normalized log likelihood as a function of time for different values of $a$, with $N = 50$ and $n = 5000$ using the \texttt{BASIC} random distribution.}
  \label{fig:ll-time-a}
\end{figure*}

\subsection{Synthetic tests}
In each experiment, we sample $n$ training sets from a base DPP of size $N$, then learn the DPP using EM and the Picard iteration. We initialize the learning process with a random positive definite matrix $L_0$ (or $K_0$ for EM) drawn from the same distribution as the true DPP kernel.

Specifically, we used two matrix distributions to draw the true kernel and the initial matrix values from:
\begin{itemize}
\setlength{\itemsep}{0pt}
\item \texttt{BASIC}: We draw the coefficients of a matrix $M$ from the uniform distribution over $\left[0, \sqrt 2\right]$, then return $L = MM^\top$ conditioned on its positive definiteness.
\item \texttt{WISHART}:  We draw $L$ from a Wishart distribution with N degrees of freedom and an identity covariance matrix, and rescale it with a factor $\frac 1 N$.
\end{itemize}

Figures~\ref{fig:ll-time-set-size},~\ref{fig:ll-time-training-size} and~\ref{fig:ll-time-a} show the log-likelihood as a function of time for different parameter values when both the true DPP kernel and the initial matrix $L_0$ were drawn from the \texttt{BASIC} distribution. Tables~\ref{table:comp-basic} and~\ref{table:comp-wishart} show the final log-likelihood and the time necessary for each method to reach 99\% of the optimal log likelihood for both distributions and parameters $n=5000$, $a=5$.

As shown in Figure~\ref{fig:ll-time-set-size}, the difference in time necessary for both methods to reach a good approximation of the final likelihood (as defined by best final likelihood) grows drastically as the size $N$ of the set of all elements $\set{1,2,\ldots,N}$ increases. Figure~\ref{fig:ll-time-training-size} illustrates the same phenomenon when $N$ is kept constant and $n$ increases.  

Finally, the influence of the parameter $a$ on convergence speed is illustrated in Figure~\ref{fig:ll-time-a}\footnote{In the cases where $a > 1$, a safeguard was added to check that the matrices returned by our algorithm were positive definite.}. Increasing $a$ noticeably increases Picard's convergence speed, as long as the matrices remain positive definite during the Picard iteration.
 
\begin{table}[htbp]
\caption{Final log-likelihoods and time necessary for an iteration to reach 99\% of the optimal log likelihood for both algorithms when using \texttt{BASIC} distribution for true and initialization matrices (training set size of 5,000, $a = 5$).}
\label{table:comp-basic}
\vskip 0.15in
\begin{center}
\begin{small}
\begin{sc}
\begin{tabular}{| l |c | c|| c |c |}
\cline{2-5}
\multicolumn{1}{l|}{} &  \multicolumn{2}{|c||}{Log-Likelihood} &  \multicolumn{2}{|c|}{Runtime to 99\%} \\ \cline{2-5}
\multicolumn{1}{l|}{} & Picard & EM & Picard & EM \\ \hline
$N=50$  & -15.5  & -15.5 & 17.3s   & 30.7s  \\ 
$N=100$ & -24.4  & -24.2 & 143s    & 75.5s  \\ 
$N=150$ & -32.5  & -32.5 & 40.7s   & 84.0s  \\ 
$N=200$ & -40.8  & -41.2 & 51.1s   & 1,730s \\ 
$N=250$	& -45.7  & -46.0 & 99.1s   & 2,850s \\  \hline
\end{tabular}
\end{sc}
\end{small}
\end{center}
\vskip -0.1in
\end{table}

\begin{table}[htbp]
\caption{Final log-likelihoods and time necessary for an iteration to reach 99\% of the optimal log likelihood for both algorithms when using \texttt{WISHART} distribution for true and initialization matrices (training set size of 5,000, $a = 5$).}
\label{table:comp-wishart}
\vskip 0.15in
\begin{center}
\begin{small}
\begin{sc}

\begin{tabular}{| l |c | c|| c |c |}
\cline{2-5}
\multicolumn{1}{l|}{} &  \multicolumn{2}{|c||}{Log-Likelihood} &  \multicolumn{2}{|c|}{Runtime to 99\%} \\ \cline{2-5}
\multicolumn{1}{l|}{} & Picard & EM & Picard & EM \\ \hline
$N=50$	& -33.0  & -33.1  & 0.2s   & 2.0s  \\ 
$N=100$	& -66.2  & -66.2  &  0.5s  & 3.6s  \\ 
$N=150$ & -99.2  & -99.3  & 0.8s  & 5.2s  \\ 
$N=200$ & -132.1 & -132.4 & 1.2s & 8.9s  \\ 
$N=250$	& -165.1 & -165.7 & 2.5s & 11s \\  \hline
\end{tabular}
\end{sc}
\end{small}
\end{center}
\vskip -0.1in
\end{table}

The greatest strength of the Picard iteration lies in its initial rapid convergence: the log-likelihood increases significantly faster for the Picard iteration than for EM. Although for small datasets EM sometimes performs better, our algorithm provides substantially better results in shorter timeframes when dealing with larger datasets.

Overall, our algorithm converges to 99\% of the optimal log-likelihood (defined as the maximum of the log-likelihoods returned by each algorithm) significantly faster than the EM algorithm for both distributions, particularly when dealing with large values of $N$.

Thus, the Picard iteration is preferable when dealing with large ground sets; it is also very well-suited to cases where larger amounts of training data are available.

\subsection{Baby registries experiment}
We tested our implementation on all 13 product categories in the baby registry dataset, using two different initializations:

\begin{itemize}
\vspace*{-15pt}
\setlength{\itemsep}{-5pt}
\item the aforementioned Wishart distribution
\item the data-dependent moment matching initialization (MM) described in~\citep{gillen}
\end{itemize}

In each case, 70\% of the baby registries in the product category were used for training; 30\% served as test. The results presented in Figures~\ref{fig:baby-test-wishart} and~\ref{fig:baby-test-mm} are averaged over 5 learning trials, each with different initial matrices; the parameter $a$ was set equal to 1.3 for all iterations.

\begin{table}[ht]
\caption{Comparison of final log-likelihoods for both algorithms; relative closeness between Picard and EM: $\delta = |\phi_{\text{em}} - \phi_{\text{pic}}|/ \phi_{\text{em}} $.}
\label{table:final-ll-comp}
\vskip 0.15in
\begin{center}
\begin{small}
\begin{sc}
\begin{tabular}{| l |c | c |}
\hline
Category  & $\delta$ (Wishart) & $\delta$ (MM) \\ \hline
furniture & 4.4e-02 & 1.2e-03 \\ 
carseats  & 3.7e-02 & 7.6e-04 \\ 
safety    & 3.3e-02 & 8.0e-04 \\ 
strollers & 3.9e-02 & 3.0e-03 \\ 
media     & 2.3e-02 & 2.4e-03 \\ 
health    & 2.6e-02 & 7.4e-03 \\ 
toys      & 2.0e-02 & 5.9e-03 \\ 
bath      & 2.6e-02 & 2.9e-03 \\ 
apparel   & 9.2e-03 & 4.3e-03 \\ 
bedding   & 1.3e-02 & 7.6e-03 \\ 
diaper    & 7.2e-03 & 5.3e-03 \\ 
gear      & 2.3e-03 & 9.0e-03 \\ 
feeding   & 4.9e-04 & 2.1e-03 \\ \hline
\end{tabular}
\end{sc}
\end{small}
\end{center}
\vskip -0.1in
\end{table}

\begin{figure*}[t]
  \centering
  \subfigure[Final negative log-likelihood]{\includegraphics[width=.45\textwidth]{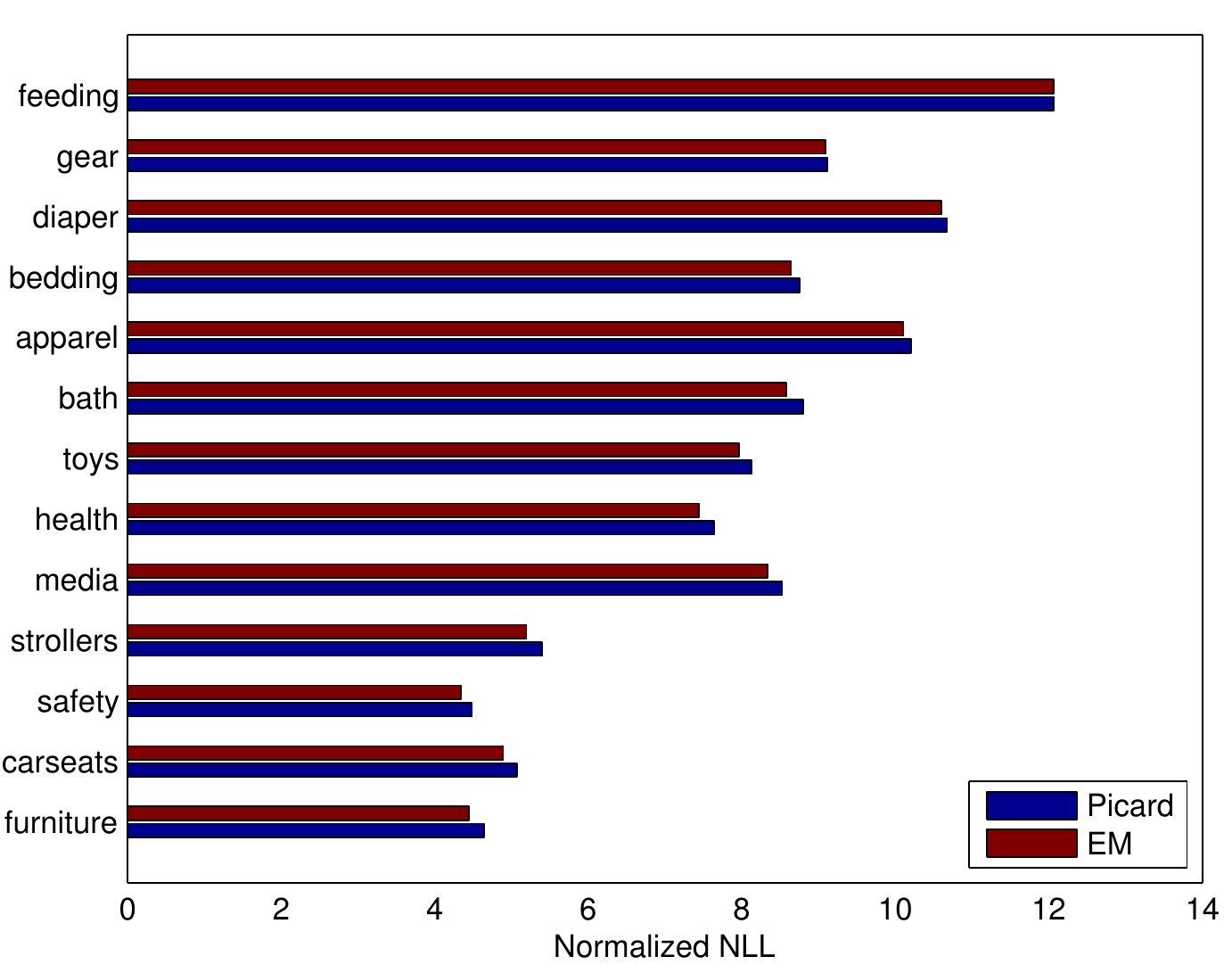}}
  \subfigure[Runtime]{\includegraphics[width=.45\textwidth]{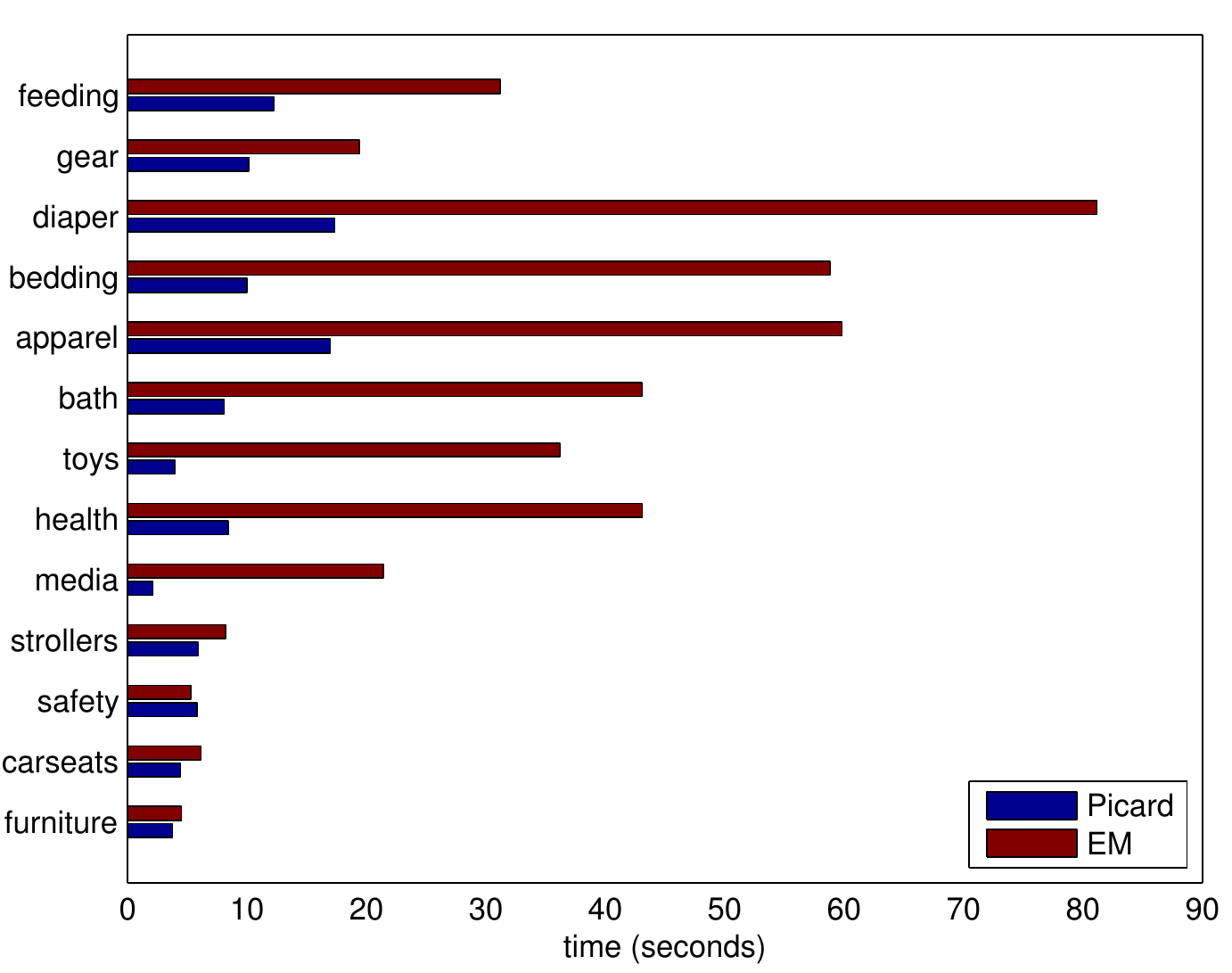}}
  \caption{Evaluation of EM and the Picard iteration on the baby registries dataset using Wishart initialization.}
  \label{fig:baby-test-wishart}
\end{figure*}

\begin{figure*}[!ht]
  \centering
  \subfigure[Final negative log-likelihood]{\includegraphics[width=.45\textwidth]{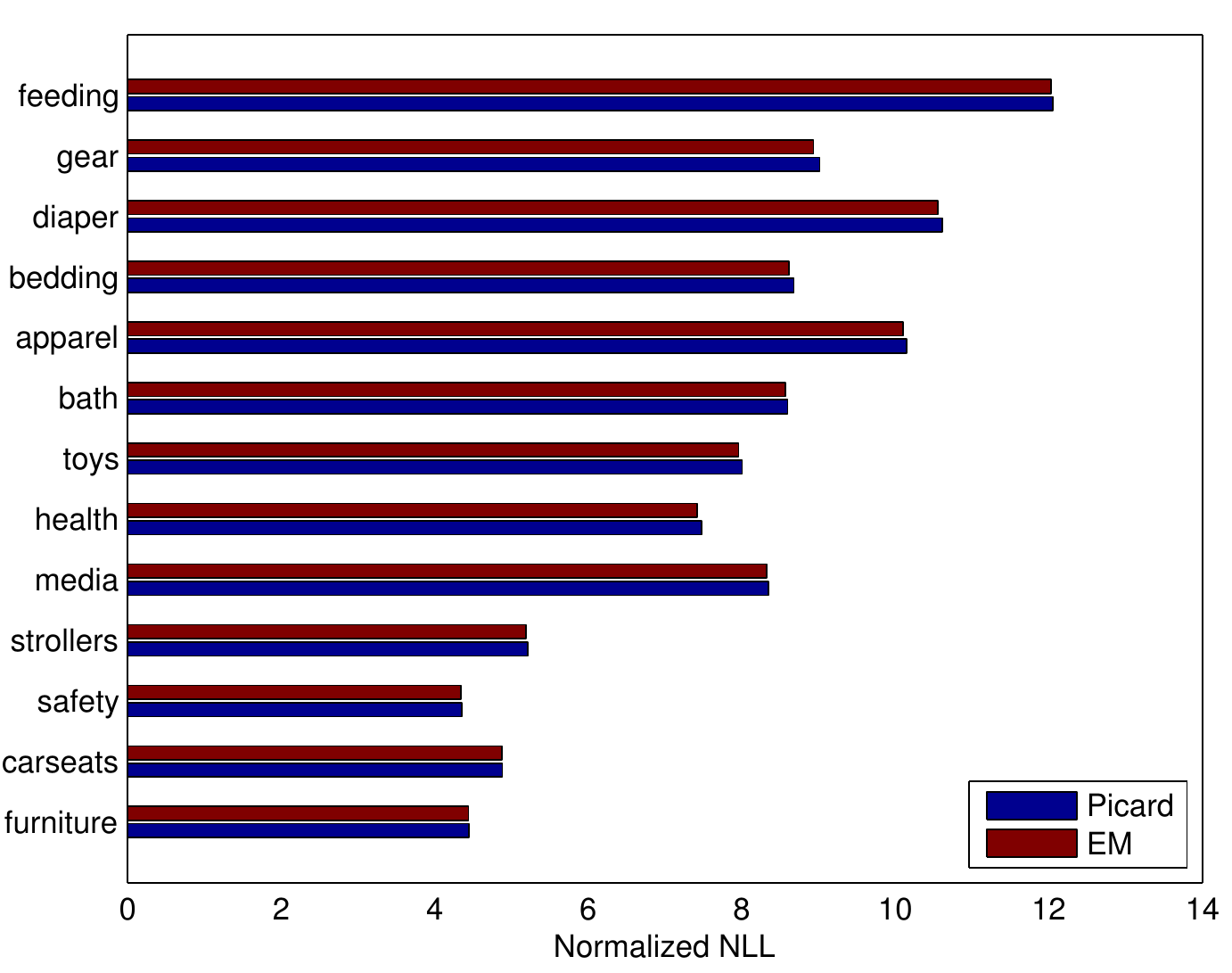}}
  \subfigure[Runtime]{\includegraphics[width=.45\textwidth]{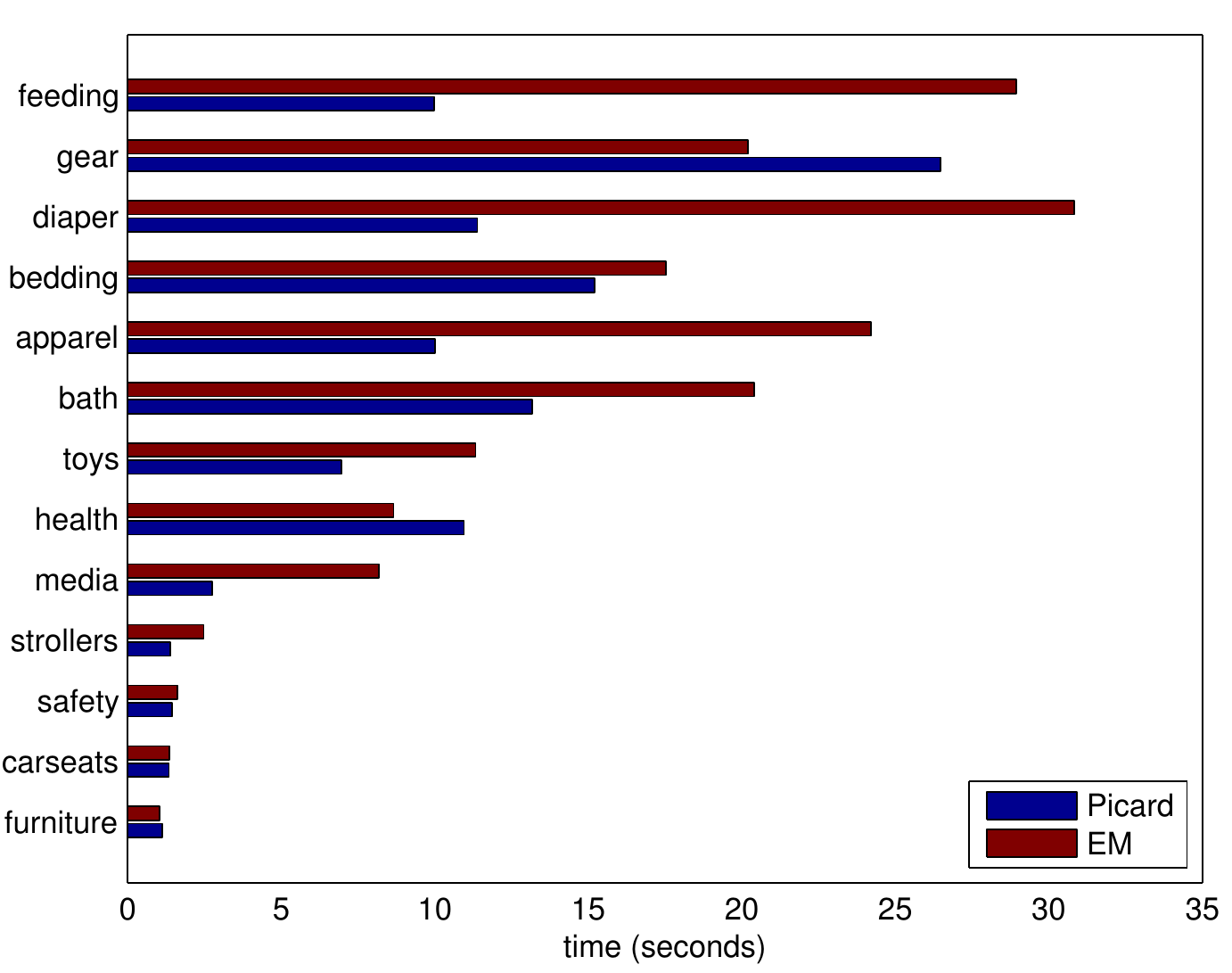}}
  \caption{Evaluation of EM and the Picard iteration on the baby registries dataset using moments-matching initialization.}
  \label{fig:baby-test-mm}
\end{figure*}

Similarly to its behavior on synthetic datasets, the Picard iteration provides overall significantly shorter runtimes when dealing with large matrices and training sets. As shown in Table~\ref{table:final-ll-comp}, the final log-likelihoods are very close (on the order $10^{-2}$ to $10^{-4}$) to those attained by the EM algorithm.

Using a moments-matching initialization leaves Picard's runtimes overall unchanged (a notable exception being the `gear' category). However, EM's runtime decreases drastically with this initialization, although it remains significantly longer than Picard's in most categories. 

The final log-likelihoods are also closer when using moments-matching initialization (see Table~\ref{table:final-ll-comp}).

\section{Conclusions and future work}
We approached the problem of maximum-likelihood estimation of a DPP kernel from a different angle: we analyzed the stationarity properties of the cost function and used them to obtain a novel fixed-point Picard iteration. Experiments on both simulated and real data showed that for a range of ground set sizes and number of samples, our Picard iteration runs remarkably faster that the previous best approach, while being extremely simple to implement. In particular, for large ground set sizes our experiments show that our algorithm cuts down runtime to a fraction of the previously optimal EM runtimes.

We presented a theoretical analysis of the convergence properties of the Picard iteration, and found sufficient conditions for its convergence. However, our experiments reveal that the Picard iteration converges for a wider range of step-sizes (parameter $a$ in the iteration and plots) than currently accessible to our theoretical analysis. It is a part of our future work to develop more complete convergence theory, especially because of its strong empirical performance. 

In light of our results, another line of future work is to apply fixed-point analysis to other DPP learning tasks.

\subsubsection*{Acknowledgments}
Suvrit Sra is partly supported by NSF grant: IIS-1409802. 

\appendix
\section{Bound on $a$}
\begin{prop}
  \label{prop.abound}
  Let $L$, $U_i$, and $\Delta$ be as defined above; set $Z = \tfrac1n\sum_i U_i\invp{U_i^*LU_i}U_i^*$. Define the constant
  \begin{equation}
    \label{eq:12}
    \gamma := \max\{\lambda_{\min}(LZ),1/\lambda_{\max}(I+L)\}.
  \end{equation}
  Then, $0 \le \gamma \le 1$ and for $a \le (1-\gamma)^{-1}$ the update
  \begin{equation*}
    L' \gets L + a L\Delta L
  \end{equation*}
  ensures that $L'$ is also positive definite. 
\end{prop}
\begin{proof}
Let $Z = \tfrac1n\sum_{i=1}^n U_i\invp{U_i^*LU_i}U_i^*$.

To ensure $L+aL\Delta L \succ 0$ we equivalently show
\begin{align*}
  &\inv{L}+ a\Bigl(\tfrac1n\sum_{i=1}^n U_i\invp{U_i^*LU_i}U_i^* - \invp{L+I}\Bigr) \succ 0\\
  &\implies\ I + a L^{1/2}ZL^{1/2} \succ a L\invp{L+I}\\
  &\implies\  (1-a)I + a(I+L)^{-1} + aL^{1/2}ZL^{1/2} \succ 0 \tag{as $L(L+I)^{-1} = I - (I+L)^{-1}$}\\
  &\implies\ (1-a) + a\lambda_{\min}(\invp{I+L}+L^{1/2}ZL^{1/2}) > 0.
\end{align*}
This inequality can be numerically optimized to find the largest feasible value of $a$. The simpler bound in question can be obtained by noting that
\begin{equation*}
  \begin{split}
    &\lambda_{\min}(\invp{I+L}+L^{1/2}ZL^{1/2}) \\
    &\ge \max\{\lambda_{\min}(LZ),1/\lambda_{\max}(I+L)\} = \gamma.
  \end{split}
\end{equation*}
Thus, we have the easily computable bound for feasible $a$:
\begin{equation*}
  a \le \frac{1}{1-\gamma}.
\end{equation*}
Clearly, by construction $\gamma \ge 0$. To see why $\gamma \le 1$, observe that $(I+L) \prec I$, so that $\lambda_{\min}(\invp{I+L}) < 1$. Further, block-matrix calculations show that $Z \preceq \inv{L}$, whereby $\lambda_{\min}(L^{1/2}ZL^{1/2}) \le \lambda_{\min}(I) = 1$.
\end{proof}

\bibliographystyle{abbrvnat}
\setlength{\bibsep}{3pt}
\bibliography{dpp}

\begin{thebibliography}{22}
\providecommand{\natexlab}[1]{#1}
\providecommand{\url}[1]{\texttt{#1}}
\expandafter\ifx\csname urlstyle\endcsname\relax
  \providecommand{\doi}[1]{doi: #1}\else
  \providecommand{\doi}{doi: \begingroup \urlstyle{rm}\Url}\fi

\bibitem[Absil et~al.(2009)Absil, Mahony, and Sepulchre]{absil}
P.-A. Absil, R.~Mahony, and R.~Sepulchre.
\newblock \emph{Optimization algorithms on matrix manifolds}.
\newblock Princeton University Press, 2009.

\bibitem[Affandi et~al.(2013)Affandi, Kulesza, Fox, and Taskar]{affKulTas}
R.~Affandi, A.~Kulesza, E.~Fox, and B.~Taskar.
\newblock {Nystr\"om approximation for large-scale Determinantal Point
  Processes}.
\newblock In \emph{Artificial Intelligence and Statistics (AISTATS)}, 2013.

\bibitem[Affandi et~al.(2014)Affandi, Fox, Adams, and Taskar]{affandi}
R.~Affandi, E.~Fox, R.~Adams, and B.~Taskar.
\newblock {Learning the parameters of Determinantal Point Process kernels}.
\newblock In \emph{International Conference on Machine Learning}, 2014.

\bibitem[Affandi et~al.(2103)Affandi, Fox, and Taskar]{affTas}
R.~Affandi, E.~Fox, and B.~Taskar.
\newblock Approximate inference in continuous {D}eterminantal {P}oint
  {P}rocesses.
\newblock In \emph{Uncertainty in Artificial Intelligence (UAI)}, 2103.

\bibitem[Bertsekas(1999)]{bertsekas99}
D.~P. Bertsekas.
\newblock \emph{{Nonlinear Programming}}.
\newblock Athena Scientific, second edition, 1999.

\bibitem[Bhatia(2007)]{bhatia07}
R.~Bhatia.
\newblock \emph{{Positive Definite Matrices}}.
\newblock Princeton University Press, 2007.

\bibitem[Boumal et~al.(2014)Boumal, Mishra, Absil, and Sepulchre]{manopt}
N.~Boumal, B.~Mishra, P.-A. Absil, and R.~Sepulchre.
\newblock {M}anopt, a {M}atlab toolbox for optimization on manifolds.
\newblock \emph{Journal of Machine Learning Research}, 15:\penalty0 1455--1459,
  2014.
\newblock URL \url{http://www.manopt.org}.

\bibitem[Gillenwater et~al.(2012)Gillenwater, Kulesza, and Taskar]{gillen12}
J.~Gillenwater, A.~Kulesza, and B.~Taskar.
\newblock Near-optimal {MAP} inference for {D}eterminantal {P}oint {P}rocesses.
\newblock In \emph{Advances in Neural Information Processing Systems (NIPS)},
  2012.

\bibitem[Gillenwater et~al.(2014)Gillenwater, Kulesza, Fox, and Taskar]{gillen}
J.~Gillenwater, A.~Kulesza, E.~Fox, and B.~Taskar.
\newblock {Expectation-Maximization for learning Determinantal Point
  Processes}.
\newblock In \emph{Advances in Neural Information Processing Systems (NIPS)},
  2014.

\bibitem[Granas and Dugundji(2003)]{granas03}
A.~Granas and J.~Dugundji.
\newblock \emph{Fixed-point theory}.
\newblock Springer, 2003.

\bibitem[Hough et~al.(2006)Hough, Krishnapur, Peres, and Vir{\'a}g]{hough}
J.~B. Hough, M.~Krishnapur, Y.~Peres, and B.~Vir{\'a}g.
\newblock Determinantal processes and independence.
\newblock \emph{Probability Surveys}, 3\penalty0 (206--229):\penalty0 9, 2006.

\bibitem[Krause et~al.(2008)Krause, Singh, and Guestrin]{krause}
A.~Krause, A.~Singh, and C.~Guestrin.
\newblock Near-optimal sensor placements in {G}aussian processes: theory,
  efficient algorithms and empirical studies.
\newblock \emph{Journal of Machine Learning Research (JMLR)}, 9:\penalty0
  235--284, 2008.

\bibitem[Kulesza(2013)]{kulesza}
A.~Kulesza.
\newblock \emph{Learning with {D}eterminantal {P}oint {P}rocesses}.
\newblock PhD thesis, University of Pennsylvania, 2013.

\bibitem[Kulesza and Taskar(2011{\natexlab{a}})]{kulTas11.icml}
A.~Kulesza and B.~Taskar.
\newblock {k-DPPs: Fixed-size Determinantal Point Processes}.
\newblock In \emph{International Conference on Maachine Learning (ICML)},
  2011{\natexlab{a}}.

\bibitem[Kulesza and Taskar(2011{\natexlab{b}})]{kulTas11.uai}
A.~Kulesza and B.~Taskar.
\newblock Learning {D}eterminantal {P}oint {P}rocesses.
\newblock In \emph{Uncertainty in Artificial Intelligence (UAI)},
  2011{\natexlab{b}}.

\bibitem[Kulesza and Taskar(2012)]{kulTas.book}
A.~Kulesza and B.~Taskar.
\newblock \emph{{D}eterminantal {P}oint {P}rocesses for machine learning},
  volume~5.
\newblock Foundations and Trends in Machine Learning, 2012.

\bibitem[Lin and Bilmes(2012)]{linBilmes}
H.~Lin and J.~Bilmes.
\newblock Learning mixtures of submodular shells with application to document
  summarization.
\newblock In \emph{Uncertainty in Artificial Intelligence (UAI)}, 2012.

\bibitem[Lyons(2003)]{lyons}
R.~Lyons.
\newblock Determinantal probability measures.
\newblock \emph{Publications Math{\'e}matiques de l'Institut des Hautes
  {\'E}tudes Scientifiques}, 98\penalty0 (1):\penalty0 167--212, 2003.

\bibitem[Macchi(1975)]{macchi}
O.~Macchi.
\newblock The coincidence approach to stochastic point processes.
\newblock \emph{Advances in Applied Probability}, 7\penalty0 (1), 1975.

\bibitem[Sra and Hosseini(2015)]{gopt}
S.~Sra and R.~Hosseini.
\newblock Conic geometric optimization on the manifold of positive definite
  matrices.
\newblock \emph{SIAM Journal on Optimization}, 25\penalty0 (1):\penalty0
  713--739, 2015.

\bibitem[Yuille and Rangarajan(2003)]{CCCP}
A.~L. Yuille and A.~Rangarajan.
\newblock The concave-convex procedure.
\newblock \emph{Neural Comput.}, 15\penalty0 (4):\penalty0 915--936, Apr. 2003.
\newblock ISSN 0899-7667.

\bibitem[Zhou et~al.(2010)Zhou, Kuscsik, Liu, Medo, Wakeling, and Zhang]{zhou}
T.~Zhou, Z.~Kuscsik, J.-G. Liu, M.~Medo, J.~R. Wakeling, and Y.-C. Zhang.
\newblock Solving the apparent diversity-accuracy dilemma of recommender
  systems.
\newblock \emph{Proceedings of the National Academy of Sciences}, 107\penalty0
  (10):\penalty0 4511--4515, 2010.

\end{thebibliography}

\end{document}